\documentclass[]{article}

\usepackage{makeidx}  

\usepackage{amsthm} 

\usepackage{latexsym}
\usepackage[pdftex]{graphics}
\usepackage{xypic}
\usepackage{amsmath}
\usepackage{amssymb}
\xyoption{all} 
\usepackage{mathrsfs}
\usepackage{eufrak}

\newtheorem{theorem}{Theorem}[section]

\newtheorem{lemma}[theorem]{Lemma}

\newtheorem{definition}[theorem]{Definition}

\newtheorem{proposition}[theorem]{Proposition}

\newtheorem{corol}[theorem]{Corollary}

\newtheorem{example}[theorem]{Example}

\newtheorem{remark}[theorem]{Remark}

\newcommand{\C}{{\mathcal C}}
\newcommand{\M}{{\mathcal M e  a  n i n g}}

\newcommand{\decode}{\rhd}
\newcommand{\code}{\lhd}

\newcommand{\dc}{{\lhd\hspace{-0.148em}\rhd}}

\begin{document}

\author{Peter Hines}
 
 \title{Types and forgetfulness in categorical linguistics and quantum mechanics}

\maketitle
\begin{abstract}
{\em 
The role of types in categorical models of meaning is investigated. A general scheme for how typed models of meaning may be used to compare sentences, regardless of their grammatical structure  
is described, and a toy example is used as an illustration. Taking as a starting point the question of whether the evaluation of such a type system `loses information', we consider the parametrized typing associated with connectives from this viewpoint.  

The answer to this question implies that, within full categorical models of meaning, the objects associated with types must exhibit a simple but subtle categorical property known as self-similarity. We investigate the category theory behind this, with explicit reference to typed systems, and their monoidal closed structure.  We then demonstrate close connections between such self-similar structures and dagger Frobenius algebras. In particular, we demonstrate that the categorical structures implied by the polymorphically typed connectives give rise to a (lax unitless) form of the special forms of Frobenius algebras known as classical structures, used heavily in abstract categorical approaches to quantum mechanics.
}
\end{abstract}

\section{Introduction}\label{intro}
A recent trend in linguistics \cite{CSC} 
is to  extend linguistic models of meaning from {\em words} to {\em sentences}. Either implicitly or explicitly, this is done via a type system -- based on a categorical grammar -- equipped with a notion of {\em evaluation}. This notion of evaluation is crucial, in that it is used to reduce all grammatically correct sentences to the same type, where they may be compared and their similarity evaluated. In this chapter, we describe this typing and evaluation process in an abstract categorical setting, based around a toy example.

We then consider the deceptively simple question of whether this evaluation process (mapping all grammatically correct sentences to entities of the same type) is reversible or not -- i.e. does evaluation lose information? Our conclusion is that in general,  forgetting information is an inevitable and crucial part of this process. However, we also demonstrate that connectives are a special case, having an entirely reversible interpretation.  Following this observation to its inevitable mathematical conclusion, we discover a connection between reversibility and polymorphic typing, in both the linguistic and logical sense. 

The relevant structures are familiar from a wide range of settings, ranging from models of lambda calculus and the Geometry of Interaction, to fractals, tilings, and the Thompson groups. This chapter demonstrates a further close connection with abstract categorical models of quantum mechanics. Precisely, we derive a (lax, infinitary)   form of the special sort of Frobenius algebras known as  {\em classical structures}, around which categorical approaches to quantum information and computation are based.

\section{Introducing typing to models of meaning}
The method of comparing meaning of words known as distributional semantics is well-known
and as such, we restrict our description to the features that will be particularly relevant to the typing process. We then give a simple example of how typing, along with an evaluation operation, is used to allow the comparison of quantities in physics. This is followed by a formal description of what we mean by a typed system, based around the theory of monoidal closed categories, and an indication of how we expect such a categorical typing in models of meaning to allow us to compare arbitrary sentences, regardless of their grammatical structure.

As described in Chapter 6, distributional semantics provides a method of associating a vector (the {\em meaning vector}) with each word in a dictionary, based on its usage in some corpus. Vectors may then be compared with each other, using any of the familiar tools from linear algebra (generally, the scalar product), giving a measure of the similarity, or overlap between words. The simple but ambitious aim is to extend this to extend this process to sentences, rather than single words, using the following scheme:
\begin{enumerate}
\item Single words are assigned {\em types}, based on their role; this typing is extended to sentences, which are typed by their grammatical structure. 
\item Associated with the type system is an evaluation, or reduction, process that reduces all grammatically correct sentences to elements of the same type (this is, as described elsewhere, a common approach in categorical linguistics).
\item Crucially, elements of the same type can be compared, providing a method of comparing the meaning of grammatically distinct sentences, in a similar way to distributional semantics.
\end{enumerate}
It hardly needs emphasising that this is a very ambitious program; instead of aiming to provide a complete or partial solution, this chapter describes features that such a  model of meaning necessarily requires, at the level of the types.

\section{What is a type?}\label{typeexample}
To a categorical logician, the answer is straightforward:  {\em a type is an object in a (monoidal, closed) category}. To explain this, we first give a simple example of typing in basic physics, followed by the formal definition,  and an illustration of why such a type system would also be useful in linguistic models of meaning.

\subsection{Types in elementary physics}
A simple, but illustrative, example of a typed system comes from basic physics, where the {\em units of measurement} may be thought of as the {\em types of quantities}. The familiar seven basic SI units (kilogram ($\bf kg$), second ($\bf s$), metre ($\bf m$), lumen ($\bf lm$), \&c.) are the fundamental types, and further types may be built up recursively, using these base types and two operations known as {\em pairing} and {\em abstraction}\footnote{Pairing and abstraction are more commonly called {\em product} and {\em quotient}. We avoid this terminology since these are neither products nor quotients in the categorical sense.}:
\begin{itemize}
\item {\bf (Pairing)} Given two types $S,T$, the pair type $ST$ may be formed. For example {\em luminous energy} is measured in {\em lumen seconds}, and hence has type ${\bf lm \ s}$.
\item {\bf (Abstraction)} Given two types $L,M$, the abstraction type $ML^{-1}$ may be formed. For example {\em velocity} is given in metres per second, and hence has type ${\bf m} \ {\bf s}^{-1}$. 
\end{itemize}
Associated with such a type system is a notion of {\em evaluation} or {\em reduction}.  A quantity of type $YX^{-1}$ may be combined with a quantity of type $X$ to return a quantity of type $Y$.  For example, let us calculate how far light, with a velocity of $c=2.997\times 10^8 {\bf ms}^{-1}$, travels in $1.3{\bf s}$. 
\begin{equation}
(2.998\times 10^8){\bf ms}^{-1}\ \times \  1.3{\bf s} \ \ = \ \ 3.897\times 10^8{\bf m}
\end{equation}
Considering the typing only, we see a reduction of the form 
\begin{equation}\label{earthtomoon}
{\bf ms}^{-1} \ \times \ {\bf s} \ \ \ \stackrel{Eval.}{\longrightarrow} \ \ {\bf m} 
\end{equation}
In this case, evaluation is simply the operation of {\em multiplication}\footnote{In general,  evaluation in a typed system may be significantly more complex; theoretical computer scientists will be familiar with evaluation as either $\beta$-reduction in lambda calculus, or the execution of a Turing machine \cite{LS,PH03}}. Thus, we observe that the type system for SI units is in fact {\em commutative} (i.e. the type $XY$ is identical to the type $YX$). As a simple consequence of this, ordering is irrelevant, and (for example) ${\bf ms}^{-1}$ is equivalent to ${\bf s}^{-1}{\bf m}$. 
In general, and in categorical linguistics in particular, neither commutativity  nor symmetry (i.e. commutativity up to isomorphism) may be assumed. To avoid ambiguity, we will therefore use type-theoretic notation, and write either $[S\rightarrow T]$ or $[T \leftarrow S]$ instead of $TS^{-1}$. Strictly, this means that we should consider two distinct evaluation operations; however, the required evaluation is often clear from the context, so for simplicity of notation we do not distinguish between the two, unless absolutely essential.

\subsection{How we wish to use types in models of meaning}\label{theaim}
By analogy with how types are used in the above simple example, we wish to consider models of meaning where words and phrases are typed according to their grammatical structure, and the evaluation operation associated with the type system reduces all (grammatically correct) sentences to the same type --- the {\em sentence type} $ S$.

Consider the simplest possible sentence structure:  
\begin{center} ({\em  Noun Phrase}\ \ , \ \  {\em  Intransitive Verb}) \end{center}
If we assume that the {\em noun phase} is of some primitive type $ NP$, an intransitive verb can only have type $[{NP} \rightarrow { S}]$, where $S$ is the sentence type. The reduction of the sentence to the type $S$ then proceeds by direct analogy with Equation \ref{earthtomoon}:

\begin{equation}\label{noun-verb}
{NP} \ \times \ {[NP\rightarrow S]} \ \ \ \stackrel{Eval.}{\longrightarrow} \ \ {S} 
\end{equation}

\subsection{Monoidal closed categories}
The above notions may be formalised in the field of category theory. We refer to Chapter 1 for the basic notions of (monoidal) category theory; however, we will be forced to take a more formal approach, and explicitly consider the structural isomorphisms:
\begin{definition}{\em Symmetric monoidal categories}\\
A {\bf monoidal category} is defined to be a category $\C$, together with a functor $\otimes : {\C}\times {\C}\rightarrow {\C}$ that satisfies, for all $A,B,C\in Ob({\C})$:
\begin{itemize}
\item {\bf Unit objects} There exists $I\in Ob({\C})$ satisfying $I\otimes A\ \cong \ A \ \cong \ A \otimes I$.
\item {\bf Associativity} 
$A\otimes (B\otimes C) \ \cong \ (A\otimes B)\otimes C$.
\end{itemize}
If a monoidal category satisfies the additional condition 
\begin{itemize}
\item {\bf symmetry}  
$A\otimes B \ \cong \ B \otimes A$.\end{itemize}
it is called a {\bf symmetric monoidal category}.
The above isomorphisms  exhibiting associativity or symmetry are  {\em natural}, and satisfy various {\em coherence  conditions} laid out in \cite{MCL}. 
\end{definition}
Due to MacLane's celebrated coherence theorem for associativity, we may treat the associativity isomorphisms as though they are strict identities ---  so we not distinguish between $A\otimes (B\otimes C)$  and $(A\otimes B)\otimes C$. We follow this practice until Section \ref{catselfsim}, where the distinction between the two will become important.

\begin{definition}\label{monoidalclosed}{\em Monoidal closed categories}\\
Let $({\C},\otimes)$ be a monoidal category. We say that it is {\bf monoidal closed} when there exists a functor 
\[ [ \ \underline{\ \ }\rightarrow  \underline{\ \ }\ ] : {\C}^{op} \times {\C} \rightarrow {\C} \]
called the {\bf internal hom} functor,
such that for fixed $B\in Ob({\C})$, the functors given by 
\[ [B \rightarrow  \underline{ \ \ }\  ]  : {\C} \rightarrow {\C}  \ \ \mbox{ and }\ \  \underline{\ \ } \otimes B : {\C}\rightarrow {\C} \]
form an adjoint pair. Equivalently, for all $X,Y,Z\in Ob(\C)$, there exists a natural isomorphism
\[ \C(X\otimes Y,Z) \ \cong \ \C(X,[Y\rightarrow Z]) \]
\end{definition}
The above definition is concise, albeit very abstract (for example, we refer to \cite{MCL} for the definition of an adjoint pair of functors). Instead we use the following characterisation that makes the existence and role of an evaluation map central:
\begin{theorem}\label{closurecurrying}
The above definition of a monoidal closed category is equivalent to the following: \\

 For every pair of objects $A,B\in Ob({\C})$, there exists 
\begin{itemize} 
\item an object $[A\rightarrow B]\in Ob(\C)$, 
\item an arrow $Eval _ {A,B} \in {\mathcal  C}( A\otimes [A\rightarrow B] , B)$
\end{itemize}
where, for all $f:A\otimes X\rightarrow B$, there exists unique $g \in {\C}(X ,[A\rightarrow B] )$ such that the following diagram commutes:
\[ \xymatrix{
A\otimes X \ar[r]^f \ar[dr] _{1_A\otimes g}  & B \\
		& A\otimes [A\rightarrow B] \ar[u]_{Eval_{A,B}}
} \]
\end{theorem}
\begin{proof} Proofs may be found in any text on category theory or categorical logic (e.g. \cite{MCL,LS}).
$\Box$
\end{proof}

\subsection{Monoidal closed categories as type systems}
The connection between the theory of monoidal closed categories, and the (very elementary) type system presented in Section \ref{typeexample} should then be straightforward. More generally, we take a categorical perspective, and {\em define} a type as an object in a monoidal closed category. The operation of pairing from Section  \ref{typeexample} is then simply the monoidal tensor $\underline{\ \ } \otimes \underline{\ \ }$, and the operation of abstraction from the same section is the internal hom functor $[ \underline{ \ \ } \rightarrow \underline{ \ \ }]$. 
Finally, the reduction operation is simply the evaluation derived in Theorem \ref{closurecurrying}. 

The question then arises: in this setting, what is an {\em quantity} of a certain type, and how may such quantities be compared?
\subsection{Elements, scalars, daggers and duals}
The objects of a monoidal closed category do not come equipped with a notion of membership, so it it not accurate to talk about ` a member $x$ of some object $N$'. Instead we have the notion of {\em elements} of an object. 
\begin{definition}\label{wellpointed}
Given a monoidal category $(\C,\otimes ,I)$, an {\bf element} of some object $N\in Ob({\C})$ is a member of $\C(I,N)$ i.e. an arrow from $I$ to $N$. The category $(\C, \otimes , I)$ is called {\bf well-pointed} when, for all $f\neq g\in\C(X,Y)$, there  exists some element $a\in\C(I,X)$ such that $fa\neq ga\in \C(I,Y)$. 
\end{definition}
For well-pointed categories, it is easy to see how the notion of elements is a reasonable replacement for the notion of membership. Most of this chapter is based on {\em elements} of a category, and their interaction with the monoidal structure, and the categories with which we work are generally {\em well-pointed}. We will point out when results depend on this assumption, or when we are (unusually) referring to a non well-pointed category.\\

\noindent
In order to {\em compare} elements of an object, we need a small amount of extra structure:
\begin{definition}
A {\bf dagger} operation on a category $\C$ is a (contravariant) involutive endofunctor, usually written $(\ )^\dagger:\C^{op}\rightarrow \C$ that is the identity on objects, so $A^\dagger=A$ for all $A\in Ob(\C)$.  An arrow $f\in \C(A,B)$ satisfying $f^\dagger f=1_A$ is called an {\bf isometry}, and when this is a two-sided inverse (so $f^\dagger$ is also an isometry), then $f$ is called {\bf unitary}.  

Let $\_ \otimes \_$ be a monoidal tensor on $\C$. 
When the monoidal structure has a well-behaved interaction with the dagger operation (that is, all canonical isomorphisms are unitary), then $(\mathcal C, \otimes ,(\ ){^\dagger})$ is called a {\bf dagger monoidal category}.
\end{definition}

Dagger monoidal categories provide us with exactly the structure we need to compare elements:
\begin{definition}
Following \cite{SAAS}, arrows from $I$ to itself in a monoidal category with daggers are called {\bf abstract scalars}. Given two elements of the same object $x,y\in \C(I,X)$, their {\bf generalised inner product} is the endomorphism of the unit object given by
\[ \langle x | y \rangle \ = \ x^\dagger y \in \C(I,I) \]
\end{definition}

Thus generalised inner products act as {\em comparisons},  give a result that is an arrow from the unit object to itself. This fits in well with our usual intuition of what it means to compare the similarity of elements, in that in various settings $\C(I,I)|$ is (for example) the real line $\mathbb R$, the complex plane $\mathbb C$, the natural numbers $\mathbb N$, the unit interval $[0,1]$, etc. We take care to avoid using categories where the endomorphism monoid of the unit object is trivial (e.g. globally defined functions, relations on sets, vector spaces with direct sum as monoidal tensor, etc.).

\begin{proposition}\label{abstractScalars}
Let $(\C,\otimes,I)$ be a monoidal category.
\begin{enumerate}
\item $\C(I,I)$ is an abelian monoid.
\item Up to canonical isomorphism, $\alpha \otimes \beta = \alpha \beta = \beta \alpha$.
\item \label{scalar_tensor} When $\C$ is a dagger monoidal category, then for all $X,Y\in Ob(\C)$ and elements 
\[ a,b\in \C(I,X) \ \ , \ \ c,d\in \C(I,Y) \]
then 
\[ \langle a\otimes c | b\otimes d\rangle \ = \ \langle a|b\rangle \ \langle c | d \rangle \]
\end{enumerate}
\end{proposition}
\begin{proof}
We refer to \cite{SAAS} for proofs.
\end{proof}
Much of the terminology and notation used in dagger monoidal categories comes from a canonical motivating example:

\begin{example}\label{HilbExample}{\em This example is based on \cite{AC}.}\\
Complex finite-dimensional Hilbert spaces form a dagger monoidal category, where the monoidal tensor is the usual tensor product, and the dagger is the usual Hermitian adjoint $( \ )^H$. The unit object is then the underlying scalar field, i.e. the complex plane $\mathbb C$, and an endomorphism of the unit object is a linear map on a one-dimensional space --- that is, multiplication by some complex scalar. 

{\em Elements} of some finite-dimensional space $H$ are then simply linear maps from $\mathbb C$ to $H$ --- which are, of course, in one-to-one correspondence with the points of $H$. Moving from points of a space to linear maps of a space is exactly the idea behind  Dirac notation for states; instead of working with the point $\psi \in H$, we work with the linear map $|\psi \rangle : \mathbb C \rightarrow H$. The (categorical) generalised inner product is then exactly the composite $|\phi\rangle ^H | \psi\rangle$, i.e. the usual inner product $\langle  \phi | \psi \rangle$ of vectors in a Hilbert space, expressed in Dirac notation.

We also refer to \cite{AC} for the monoidal closure of this category, and the quantum-mechanical interpretation of the categorical operations such as evaluation and the dagger. 
\end{example}

In the above example, the generalised inner product is exactly the scalar product of vectors, and may be used to define a {\em metric} on elements of a (finite-dimensional) Hilbert space. Thus, because of the first metric axiom, the generalised inner product may be used as a test of equality for elements.  In other examples, (such as partial reversible functions on sets), the endomorphism monoid of the unit object is trivial, and the generalised inner product provides little or no information about elements. 
We axiomatise this distinction as follows: 
\begin{definition}\label{discrim}
Let $(\C,\otimes ,I (\ )^\dagger)$ be a dagger monoidal category. We say that $(\ ):\C^{op}\rightarrow \C$ {\bf discriminates elements} of $A\in Ob(\C)$ when, for all $x,y\in \C(I,A)$, 
\[  \langle x | y \rangle =1_I \ \ \Leftrightarrow \ \ x=y\in \C(I,A) \]
When $(\ )^\dagger:\C^{op}\rightarrow \C$ discriminates elements of all objects of $\C$, we simply say that it  {\bf discriminates elements}.
\end{definition}

\subsection{Elements, names, and evaluation}
In a monoidal {\em closed} category, the elements of the object $[X\rightarrow Y]\in Ob({\C})$ have a natural interpretation as arrows from $X$ to $Y$ within $\C$, as the following result makes clear: 

\begin{proposition}\label{naming}
Let
$(\C,\otimes,[ \ , ] , I)$ be a monoidal closed category. Then for all objects $X,Y\in Ob(\C)$, there is a natural bijection between elements of $[X\rightarrow Y]$, and the homset $\C (X,Y)$.  
\end{proposition}
\begin{proof}This is a standard result from the theory of closed categories and categorical logic \cite{ML,LS}. 
\end{proof}

\begin{definition}
Given a monoidal closed category $(\C,\otimes,[ \ , ] , I)$, and an arrow $f\in \C(X,Y)$, then its image under the bijection  of Proposition \ref{naming} above is called the {\bf name} of $f\in \C(X,Y)$, written $\ulcorner f \urcorner \in \C(I,[X\rightarrow Y])$.
\end{definition}

The intuitive meaning of evaluation is that it `promotes' an {\em element} (i.e. the name of an arrow $\ulcorner f \urcorner$) to actual {\em arrow} within the category; more formally, the following diagram is a special case of the diagram of Theorem \ref{closurecurrying}.

\[ \xymatrix{
A \ar[rr]^f && B \\
A \otimes I \ar@{-}[u]^\cong \ar[rr]_<<<<<<<<<{1_A \otimes \ulcorner f \urcorner } && A \otimes [A\rightarrow B] \ar[u]_{Eval_{A,B}}
} \]
Thus, for example, we see that an element of the `intransitive verb' object (as in Section \ref{theaim}) may also be considered as an arrow from the `noun phrase' object to the `sentence' object.

Given that {\em elements} of an object are themselves arrows in a category, it is natural to wonder what the name of an element is, or indeed the name of the name of an element, etc. Fortunately, such an eternal recurrence is avoided by the fact that, up to canonical isomorphism, elements `name themselves'.  Precisely, for any element $x\in {\C}(I,A)$ in a monoidal closed category, the following diagram commutes:
\[ 
\xymatrix{
														& I 	\ar@{-}[dl]_{\cong}	\ar[dr]^x			&	\\
I \otimes I	\ar[r]|{1_I\otimes x}	\ar[dr]_{1_I\otimes \ulcorner x \urcorner} 	& I\otimes A	\ar@{-}[r]|\cong				& A	\\
														& I\otimes [I\rightarrow A]	\ar[ur]_{Eval_{I,A}}	&	\\
} \]
{\bf Convention} In the above diagram, lines denoting canonical coherence isomorphisms are simply labelled by ``$\cong$''. We follow this convention throughout, unless the precise coherence isomorphism is important.

\subsection{Types for linguistics and models of meaning}

In categorical linguistics, a common method of characterising grammatically correct sentences is to assign types (i.e. objects in a monoidal closed category) to words in such a way that the evaluation map takes all grammatically correct sentences to a single distinguished type $S$, called the {\em sentence type}. In particular as shown in \cite{CSC}, this is the structure behind Lambek's {\em pregroup semantics} and other approaches to categorical linguistics. Thus, standard categorical models of linguistics provide a type system for models of meaning applicable to arbitrary sentences; however, we do not yet have actual elements, or any notion of how these interact with the evaluation process, the generalised inner product, or the monoidal tensor. An obvious analogy exists with the very powerful tool of {\em dimensional analysis} in basic physics \cite{KLR}, which may be considered to be the underlying type system behind the SI units of Section \ref{typeexample}, abstracted from consideration of actual quantities.

The remainder of this paper may be considered as an investigation of what it would mean to (re)introduce actual elements to the type system provided by categorical linguistics, and indeed what modifications must be made to the typing to account for the fact that we are interested in meaning as well as grammar.

\subsection{Typed models of meaning --- 
a toy example}\label{toyexample}
In order to avoid becoming too abstract, we use a concrete example to illustrate how the program described above may be used to compare two sentences. We will use the following examples:
\begin{itemize}
\item[$L1$]  {\em Bobby loves Marilyn Monroe.}
\item[$L2$] {\em I like Fidel Castro and his beard.}
\end{itemize}
The classically educated reader will recognise these as lyrics from Bob Dylan songs.
(Note that one of these lyrics is from an improvised live performance and is not part of the official Dylan canon \cite{BD}); our interest is in how a typed model of meaning could be used to compare these distinct lyrics. 

 The first step we take is to instantiate the variable\footnote{Note that this is an {\em exophoric reference} since the pronoun {\em I} is not bound to any noun phrase within the text itself.} in $L2$; as we are familiar with these sentences as Bob Dylan lyrics, it is reasonable to replace {\em I} by {\em Bob Dylan}, and adjust the verb from the first to the third person, giving 
\begin{itemize}
\item[$L2'$] {\em Bob Dylan likes Fidel Castro and his beard.}
\end{itemize}
We draw these two sentences $L1$ and $L2'$ in tree form as shown in Figure \ref{sentencetrees}, and consider how both the individual constituents and the sentences as a whole may be compared.
\begin{figure}[h]\label{sentencetrees}\begin{center}\caption{Trees for $L1$ and $L2'$}
\scalebox{0.8}{ \xymatrix{
	 		{ L1}						&	*+[F:<3pt>]\txt{loves} \ar@{-}[dl]	\ar@{-}[dr]		& 											&    \\
			*+[F:<3pt>]\txt{Bobby}		& 										& *+[F:<3pt>]\txt{Marilyn \\ Monroe } 					&  \\
									&										&											& \\
	 		{L2'}						&	*+[F:<3pt>]\txt{likes} \ar@{-}[dl]	\ar@{-}[dr]		& 											&    \\
			*+[F:<3pt>]\txt{Bob \\ Dylan}	& 										& *+[F:<3pt>]\txt{and } \ar@{-}[dl]	\ar@{-}[dr]			&  \\
									& 	 *+[F:<3pt>]\txt{Fidel \\ Castro}				& 											&  *+[F:<3pt>]\txt{his \\ beard} \\
} }
\end{center}
\end{figure}
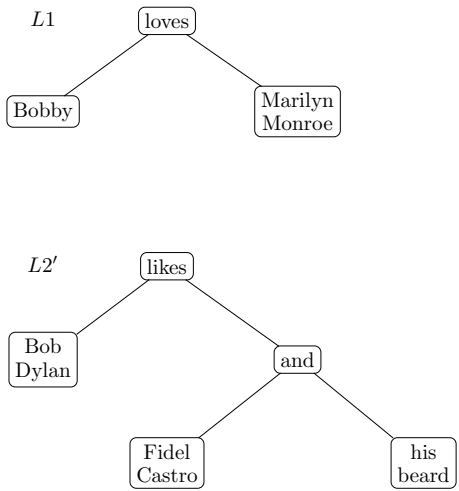

From these trees, we see that the basic grammatical components we require are:
\begin{itemize}
\item {\bf Sentences (S)} {\em `Bob Dylan likes Fidel Castro and his beard'},  {\em `Bobby loves Marilyn Monroe'}.
\item {\bf  Transitive Verbs (TV)} {\em likes}, {\em loves}.
\item {\bf Connectives (C)} {\em and}.
\item {\bf noun phrases (NP)} {\em Bob Dylan}, {\em Bobby}, {\em Marilyn Monroe}, {\em Fidel Castro}, {\em Fidel Castro's beard}, {\em Fidel Castro and his beard}.
\end{itemize}
(We observe above that joining two noun phrases by a connective forms another noun phrase. This  possibility must be reflected in the typing associated with connectives; this is explored further in Section \ref{connectivetyping} onwards).

\subsection{Categorical features for a model of meaning}
We now consider the requirements for some category $\M$ in which the meanings of $L1$ and $L2'$ may be evaluated and compared. We do not present a concrete example; rather we use machinery developed to lay down requirements that such a category must satisfy, and go on to consider the resulting categorical theory.

For a type system that allows us to model grammatical structure and reduce grammatical sentences to elements of the same type, we require a {\em monoidal closed} structure, so $\M$ is equipped with a monoidal tensor $(\_ \otimes \_ ): \M \times \M \rightarrow \M$, a unit object $I\in Ob(M)$,  and an internal hom, $[\_ \rightarrow \_] : \M^{op} \times \M \rightarrow \M$. In order to compare elements, we will also require a dagger operation $( \_ )^\dagger : \M^{op}\rightarrow \M$, compatible with the monoidal tensor, giving  a generalised scalar product.

As our analysis will be based on elements of objects, it is reasonable to assume that $(\M,\otimes ,  [\underline{ } \rightarrow \underline{\ }], I , (\ )^\dagger)$ is well-pointed (Definition \ref{wellpointed}).  We further assume that our model of meaning is {\em complete}, in the sense that distinct concepts are not unnecessarily identified by the generalised scalar product\footnote{This is not always a reasonable assumption to make, depending on the intended purpose of our model of meaning. In particular, the very successful field of {\em sentiment analysis} \cite{KH} makes an entirely different assumption!}; categorically, this requires that 
the dagger operation {\em discriminates elements}, in the sense of Definition \ref{discrim}.

The category $\M$ must also contain objects corresponding to the grammatical components given in Section \ref{toyexample} above. Thus $NP,C,TV,S\in Ob(\M)$ are the objects corresponding to the {\em noun phrase}, {\em connective}, {\em transitive verb}, {\em sentence} types; we take the noun phrase and sentence types $NP,S\in Ob(M)$ as primitive and build up the others in terms of their desired behaviour under evaluation. 

Finally, for illustrative purposes, we take $\M (I,I)$ to be the unit interval $[0,1]$. The putative interpretation is that $\langle x|y\rangle = 1$ means complete equality of meaning between elements $x$ and $y$, whereas $\langle x|y\rangle =0$ means that they have nothing in common. Composition of endomorphism arrows of the unit object is, as per the requirements of Proposition \ref{abstractScalars}, simply multiplication.

\subsection{Comparing simple nouns}\label{SNsp}
Let us start with the respective subjects of $ L1$ and $L2'$, the noun phrases {\em Bobby} and {\em Bob Dylan}. In a suitable typed system, these will be represented by two distinct elements of type $NP$
\[ 
\xymatrix{
I \ar[r]^{\mbox{Bobby}} & NP & & I \ar[rr]^{\mbox{Bob Dylan}}& & NP } \]
These elements may be compared by computing their generalised inner product, giving  
\[ 
\xymatrix{
I \ar[rr]^{\mbox{Bobby}} \ar[drr]_<<<<<<<<{\langle \mbox{Bobby} | \mbox{Bob Dylan}\rangle\ \ \ } 	&& NP  \ar[d]^{\mbox{Bob Dylan}^\dagger} \\ 
																&& I 
 } \]

Although this chapter does not present a concrete model of meaning, observe that the above comparison would be straightforward, using the distributional semantics approach described in Chapter 6. In the absence of any concrete data, we  make a guess for illustrative purposes, and write 
\[ \langle Bobby | Bob\ Dylan\rangle \ = \ 0.98 \in \M(I,I) \]
Giving a high, if not perfect overlap\footnote{Readers familiar with $20^{\mbox{th}}$ Century US culture might assume that {\em Bobby} in $L1$ instead refers to Robert Kennedy. Whether or not this interpretation is correct, it is neither apparent from the individual lines, nor the songs as a whole. Interpreting texts in their appropriate historical and cultural context is a significant challenge for models of meaning generally.}  between {\em Bobby} and {\em Bob Dylan}.

\subsection{Comparing transitive verbs}\label{TVsc}
We now compare the central verbs of $L1$ and $L2'$, i.e. we wish to assign a value to the generalised inner product $\langle \mbox{\em likes}|\mbox{\em loves}\rangle$. However, it is worth considering the typing that these elements must have. Recall from Section \ref{theaim} that an {intransitive} verb can only have type $[NP\rightarrow S]$. Thus, we wish a {transitive} verb to have a suitable type so that, when given an object (i.e. a noun phrase) on its right hand side, it returns something of type $[N\rightarrow S]$. Thus, an intransitive verb must have type $[[NP\rightarrow S] \leftarrow NP]$. 

The comparison of {\em likes} and {\em loves}  is the following generalised inner product 
\[ 
\xymatrix{
I \ar[rr]^<<<<<<<<<<<{\mbox{loves}} \ar[drr]_{\langle \mbox{likes} | \mbox{loves}\rangle\ \ \ } 	&& [[N\rightarrow S]\leftarrow N]  \ar[d]^{\mbox{likes}^\dagger} \\ 
																&& I 
 } \]

We again make an arbitrary guess\footnote{Based on a talk by Mehrnoosh Sadrzadeh (Oxford, Oct. 2010),  where denotational semantics was introduced using the illustration that ``{\em likes} is $\frac{3}{4}$ {\em loves} and $\frac{1}{4}$ {\em hates}''.} and write 
\[ \langle likes | loves\rangle \ = \ 0.75 \in \M(I,I) \]

\subsection{Comparing noun phrases}\label{NPsc}
In terms of comparing the primitive elements of $L1$ and $L2$, it now remains to compare the two objects of the transitive verbs: {\em  Marilyn Monroe}, and {\em Fidel Castro and his beard}. Leaving aside for the moment the details of how two noun phrases may be combined with a connective to produce a further noun phrase, we are happy to declare that there never has been any significant overlap between {  \em  Marilyn Monroe} and {  \em Fidel Castro and his beard}. Thus, our educated guess at this point is simply that 
\[ \langle Marilyn\ Monroe| Fidel\ Castro\ and\ his\ beard\rangle \ = \ 0 \in \M (I,I)  \]

\section{Comparing words vs. comparing sentences}\label{wordsentence}
Bringing together the  (entirely fictitious) values for the overlap between the meanings of words introduced above, we have the table shown in Figure \ref{wordcomparisons}.

\begin{figure}[b]\label{wordcomparisons}\begin{center}\caption{Comparisons of individual words in $  L1$ and $  L2$}

\scalebox{0.8}{
\begin{tabular}{c|c|c}
& & \\
Bobby & loves & Marilyn Monroe. \\
  Bob   Dylan & likes & Fidel Castro and his beard. \\
  \hline
$\langle Bobby | Bob \ Dylan\rangle$  & $\langle likes | loves \rangle$  & $\langle Fidel \ \& \ his \ beard | Marilyn \ Monroe \rangle $   \\
$=$ & $=$ & $=$ \\
$0.98$ & $0.75$  & $ 0.00$ \\
\end{tabular}
}
\end{center}
\end{figure}

The crucial question is whether these three values are enough to compare the meaning of $  L1$ and $  L2'$ ? We first appeal to part \ref{scalar_tensor} of Proposition \ref{abstractScalars}, we may compute the generalised inner product of $  L1$ and $  L2'$, {\em considered as elements of type $NP \otimes [ [NP\rightarrow S]\leftarrow NP] \otimes NP$}.

\begin{proposition} Using the values for the generalised scalar products of individual word proposed in Sections \ref{SNsp} - \ref{NPsc}, the inner product of the elements 
\[ {\bf  L1} \ : \ I \rightarrow \ NP \otimes [ [NP\rightarrow S]\leftarrow NP] \otimes NP \]
\[ {\bf L2'} \ : \ I \rightarrow \ NP \otimes [ [NP\rightarrow S]\leftarrow NP] \otimes NP \]
is exactly $\langle {  L1}|{  L2'}\rangle = 0\in \M(I,I)$.
\end{proposition}
\begin{proof}
This follows from the values given in Figure \ref{wordcomparisons}, and part \ref{scalar_tensor} of Proposition \ref{abstractScalars}, where the interaction of generalised inner products and monoidal tensors is given.
\end{proof}
However,  we have compared these sentences {\em before any evaluation has taken place} --- and the whole point of the typing system was that all well-formed sentences evaluate to the same sentence type $ S$. The key question is then whether this matters, i.e.
\begin{center}{\em Is comparison of sentences invariant under evaluation?}\end{center}

\section{Inner products, evaluation, and inverses}\label{SPpreserving}
The question at the end of Section \ref{wordsentence} above should properly be considered as two distinct questions: 
\begin{enumerate}
\item\label{untyped} Does the evaluation arrow $Eval_{A,B}=\in \C(A \otimes [A\rightarrow B] , B)$ preserve inner products?
\item\label{forgetfulnames} When the meaning of a word is some name $\ulcorner f \urcorner \in \C(I,[X\rightarrow Y])$, does the arrow $f\in \C(X,Y)$ preserve inner products?
\end{enumerate}
Question \ref{untyped} is a fundamentally category-theoretic question, whereas question 
\ref{forgetfulnames} is about how we expect categorical models of meaning to behave. 
In a dagger monoidal closed category $(\C , \otimes , [ \_ \rightarrow \_ ], (\ )^\dagger)$, both isometries and unitaries preserve generalised inner products, and the canonical isomorphisms for the monoidal structure are unitary. Therefore, any dagger monoidal category contains inner product preserving arrows -- question \ref{forgetfulnames} is simply asking whether any of these have a role to play in models of meaning.

As we are working within a well-pointed monoidal category with a dagger that discriminates elements, both these questions are about  whether various categorical operations `lose information', as the following result demonstrates:
\begin{lemma}\label{IPPinverse}
Let $(\mathcal C,\otimes ,(\ )^\dagger)$ be a well-pointed dagger monoidal category where the dagger discriminates elements, and let $F\in \C(A,B)$ preserve generalised scalar products, Then $F$ is an isometry, i.e. $F^\dagger\in\C(B,A)$ is a left inverse of $F\in \C(A,B)$.
\end{lemma}
\begin{proof}
Consider arbitrary elements of $x,y\in \C(I,A)$. As $F$ preserves inner products, the following diagram commutes:
\[  \xymatrix{
A 	\ar[r]^F			& B \ar[r]^{F^\dagger}	& A \ar[d]^{y^\dagger}	 \\
I \ar[u]^x \ar[r]^x	& A \ar[r]^{y^\dagger} 	& I 	\\
} \]
Simplifying this commuting diagram, we have
\[  \xymatrix{
 		& B \ar@/^6pt/[d]^{F^\dagger}	& 	 \\
I \ar[r]^x 	& A\ar@/^6pt/[u]^F \ar[r]^{y^\dagger} 	& I 	\\
} \]
Thus, as $\C$ is well-pointed with a dagger that discriminates elements, we deduce that $F^\dagger$ is a left inverse of $F$.
\end{proof}

Note that the above result does not prove that $F^\dagger$ is a two-sided inverse; indeed, in arbitrary Hilbert spaces, the inner-product preserving isomorphisms are exactly the unitary maps, whereas inner product preserving linear maps are simply isometries (which do indeed have a left inverse, but not necessarily a two-sided inverse).

\section{Does evaluation preserve inner products?}
We first address Question \ref{untyped} of Section \ref{SPpreserving} above: is the evaluation map an isometry -- i.e. is its dagger also a left inverse?

Given elements $x,y$ of an object $A \otimes [A\rightarrow B]$ in some monoidal closed category with a discriminating dagger, we may form elements of $B$ by composing both $x$ and $y$ with  the canonical evaluation map $Eval_{A,B}:A \otimes [A\rightarrow B]\rightarrow B$ as shown below:
\begin{center}
\scalebox{0.8}{ \xymatrix{
I \ar[rr]^{x}\ar[ddrr]_{Eval_{A,B}  x} 	& & A \otimes [A\rightarrow B] \ar[dd]|{Eval_{A,B}}& &  I \ar[ll]_y \ar[ddll]^{Eval_{A,B} y}  \\ 
						& & 										\\ 
						& & B 									 \\ 
} 
}
\end{center}
From Lemma \ref{IPPinverse}, a necessary condition for the the following diagram to commute
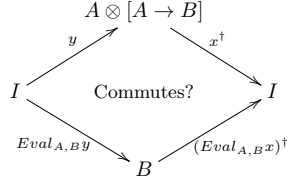
\begin{center}
\scalebox{0.8}{ \xymatrix{
& 							 	A \otimes [A\rightarrow B] \ar[dr]^{x^\dagger} 				& \\ 
  I 	\ar[ur]^{y} \ar[dr]_{Eval_{A,B} y}	&   {\mbox{\small Commutes?}} 	&  I \\ 
		 				&	B  \ar[ur]_{(Eval_{A,B}x)^\dagger}		& \\ 
} 
}
\end{center}
is that $Eval_{AB}$ has a left inverse. Leaving aside the irrelevant (for our purposes) case where $\C$ does not have a discriminating dagger, in general the above diagram does {\em not} commute. One of the simplest counterexamples is the motivating example of Example \ref{HilbExample}, and quantum-mechanical interpretations of its categorical properties, where Evaluation interprets as (post-selected partial) measurement against a maximally entangled basis \cite{AC}. Of course, one of the most fundamental features of the Hilbert space model of quantum mechanics is that measurement (partial or total) is certainly not a reversible operation. Other examples include models of logic or lambda calculus, where evaluation is either $\beta$-reduction, or cut-elimination --- neither of which are reversible operations\footnote{We do not claim that in any monoidal closed category, the evaluation arrow {\em cannot} be invertible. In particular, monoidal closed categories of partial reversible functions, where all objects are isomorphic,  have been constructed in \cite{PH98,PH99,AHS}. Leaving these rather esoteric examples aside, our claim is that a (much more usual) irreversible evaluation is highly desirable and useful for typed models of meaning}. 

The question is whether this is {\em desirable} or {\em undesirable} for a model of meaning? The linguistic justification for the answer to the second question of Section \ref{SPpreserving} helps demonstrate that it is in fact desirable.

\subsection{Forgetfulness - a linguistic justification}\label{namedforgetting}
We now address Question \ref{forgetfulnames} from Section \ref{SPpreserving}. The claim that we make is that it is {\em vital} for the evaluation process to be {\em irreversible}, since we need it to forget information ---  it is highly desirable that the arrows named by elements in our models of meaning do not have inverses.

As a motivating example, consider the simple noun phrase {  \em  scruffy cats}, built up from an adjective and another noun phrase:
\[ \xymatrix{ 
I \ar[r]^{   scruffy} & AD   &  &  I \ar[r]^{  cats} & NP \\
 } \]
The noun phrase {  \em cats} is a simple element of the object $  NP$, and from its behaviour we deduce that an adjective has typing $  AD=[NP \leftarrow NP]$. The term {  \em  scruffy cats}, before any reduction, is therefore is the following element:
\[ \xymatrix{
I \ar[rr]^<<<<<<<<<{   scruffy\ \otimes \  cats} & & [NP \leftarrow NP]\otimes NP \\
 }
\]
Momentarily forgetting about typing questions, let us assume that the `meaning' of both {  \em  scruffy} and {  \em cats} has been derived using some variant of the distributional semantics described in Chapter 6.  The `meaning' of {  \em cats} will then provide information about cats generally, whether scruffy, tidy, or invisible. Similarly, the adjective {  \em  scruffy} provides information about the general concept of scruffiness, whether applied to cats, dogs, or academics.

From Proposition \ref{naming}, an element ${   scruffy}\in \M(I, [NP\leftarrow NP])$ 
is the name of some arrow $\widetilde{   scruffy} \in \M(NP,NP)$. We then see that, at least in this setting,  the arrow named by ${   scruffy}\in \M(I,[NP\leftarrow NP])$ has something of the nature of a projector, or a partial identity, in that it acts to restrict a concept to a special case.

The above is not, of course, a formal justification. However, we also observe that reduction is often a multi-stage process, and the ability to compare sentences or sentence fragments at different levels may be a highly useful feature of a typed model of meaning.  Consider sentences $s_0,t_0$ of some compound type $[S\leftarrow X]\otimes Y\otimes [Y\rightarrow X]$. This compound type may be reduced to $S\in Ob(\M)$ in two stages.
\[ 
\xymatrix{
							&	& [S\leftarrow X]\otimes Y \otimes [Y\rightarrow X] \ar[d]^{1\otimes Eval_{Y,X}} \\
I \ar[drr]_{s_2,t_2}	\ar[rr]|{s_1,t_1}	 \ar[urr]^{s_0,t_0} &	& [S\leftarrow X]\otimes X \ar[d]^{Eval_{X,S}}\\
							&	& S 
} 
\]
given that $\langle s_0 | t_0 \rangle \neq \langle s_1 | t_1 \rangle \neq \langle s_2 | t_2 \rangle $, we observe that it is possible to compare sentences {\em at many different levels}, depending on how much reduction has been carried out. This unusual feature may prove useful in dealing with ambiguity, or indeed in assigning meaning to non-compositional phrases such as {  \em Iron Curtain}, where the meaning of this phrase is not derived by restricting the information about all possible curtains to those made of iron.

\section{How to type connectives?}\label{connectivetyping}
We have taken a digression in our aim of comparing the meaning of two distinct Bob Dylan lyrics; in particular, we left the question of how to deal with connectives unanswered.  This was intentional, in that -- as we demonstrate below -- the behaviour of connectives is closely connected with questions of reversibility and evaluation. 

Recall that we treated the noun phrases
\begin{enumerate}
\item {\em  Marilyn Monroe}
\item\label{hairyfidel}{\em  Fidel Castro and his beard}
\end{enumerate}
simply as two distinct noun phrases. However, \ref{hairyfidel} above is clearly the {\em conjunction} of two distinct noun phrases; rather than being a noun phrase itself, it is a compound that should {\em evaluate to} a noun phrase. The question then, is simply, how should we type {\em and} ?
As the typing will prove rather intricate, we first consider an alternative method of dealing with connectives:

\subsection{Distributivity and conjunction}
A common point of view is that, given a sentence containing the conjunction of two noun phrases, it should simply be split in two using {\em distributivity},  and the two sentences treated separately. For example, using distributivity, 
\begin{itemize}
\item[L2]{  \em Bob Dylan likes Fidel Castro and his beard}
\end{itemize}
would be replaced by
\begin{itemize}
\item[L$2_a$]{  \em Bob Dylan likes Fidel Castro.}
\item[L$2_b$]{  \em Bob Dylan likes Fidel Castro's beard.}
\end{itemize}
This seems to be valid from a grammatical point of view, and (assuming we  resolve the anaphor {  \em his} before applying distributivity) the meaning of $L1$ is indeed the conjunction of $L1_a$ and $L1_b$. However, this is not always the case. Consider the following sentence:
\begin{itemize}
\item[T] {  \em Fidel Castro and Marilyn Monroe played tennis}
\end{itemize}
Applying distributivity, we get the (grammatically correct) 
\begin{itemize}
\item[$T_a$] {  \em Fidel Castro played tennis.}
\item[$T_b$] {  \em Marilyn Monroe played tennis.}
\end{itemize}
Intuitively, we are happy to believe the conjunction of $T_a$ and $T_b$, but find  $T$ rather implausible -- since tennis is generally an activity indulged in by two people, we deduce that it was a joint, shared game of tennis. 

Although the above example is somewhat facetious, the question of when and whether applying distributivity changes meaning has been heavily studied \cite{YW}, including in a legal context  \cite{CV}. See \cite{TH} for a particular case involving arguments on whether distributivity is applicable to conjunction in the phrase {\em ``to keep and bear arms"}, and whether doing so changes the meaning of this phrase.
It appears that, when we consider meaning as well as grammatical correctness, we are forced to consider how the connectives ({  \em and}, {  \em or}, etc.) are typed, and behave under evaluation.

\subsection{Typing connectives and polymorphism}
The first problem is that although (based on its usage in $  L1$) we might simply wish to type {  \em and} as an element of 
$[[NP \rightarrow NP]\leftarrow NP]$ (or equivalently, $[NP \rightarrow [NP\leftarrow NP]]$), the word `and'  is used in other settings, as Figure \ref{conjunctions} demonstrates.

\begin{figure}[t]\label{conjunctions}\begin{center}\caption{Conjunction as a polymorphic connective}\end{center}

\begin{tabular}{|r|l|}
\hline
\multicolumn{2}{|c|}{} \\
\multicolumn{2}{|c|}{\large Conjunction in different contexts} \\
\multicolumn{2}{|c|}{} \\
 \hline
\hline
	& \\ 
Noun phrases  & {  \em Fidel Castro and his beard} 	\\
	& \\ 
\hline
	& \\ 
Transitive verbs  & {  \em Bobby loves and obeys Marilyn Monroe}	\\
	& \\ 
\hline
	& \\ 
Adjectives & {  \em Fidel's big and bushy beard}					\\
	& \\ 
\hline
	& \\ 
Sentences & {  \em Bobby likes Fidel and I like Marilyn Monroe}		\\
	& \\ 
\hline
\end{tabular}
\end{figure}

However, in every case, the appropriate typing appears to be 
\[ [[X \rightarrow X]\leftarrow X] \]
where $X$ ranges over types, according to context. The same phenomenon appears to apply to other binary connectives\footnote{e.g. $  or$ may be substituted for $  and$ in any of the above. Also, although English does not have a single connective corresponding to `exclusive or', one could easily conceive of sentences such as {  \em I like exactly one of Fidel Castro and his beard}, which would behave in a similar way. However, the same does not hold for {\em implies}, which is generally applied to entire sentences only. }
Thus, it appears that the typing of binary connectives is {\em polymorphic}. We refer to \cite{JR} for the notion of parametrised types in computer science, and \cite{JYG} for {\em System F}, the polymorphic lambda calculus. Borrowing notation from this polymorphic lambda calculus, we write the type of $ and$ as
\[ \Lambda X . [[X \rightarrow X]\leftarrow X] \]
We do not give a full treatment in terms of the polymorphic lambda calculus; rather we simply treat this as shorthand for the following: Given some binary connective $B$, then the type of $B$ is dependent on the context; given some element of type 
\[ U \otimes B \otimes V \]
together with evaluation arrows $Eval_{U,X}$ and $Eval_{V,X}$, then the `polymorphically typed'  connective $B$ provides us with some element $B_X$ of type 
\[ X\rightarrow [X\leftarrow X] \ \ \mbox{ or equiv. } \ \ [X\rightarrow X]\leftarrow X \]

\subsection{Forgetfulness and binary connectives}\label{forgettingAnd}
In Section \ref{namedforgetting}, we make an argument, based on linguistic interpretation, that the arrows of a category named by word of various types are {\em forgetful} --- they lose information, the example given being how adjectives should, in certain cases, act as projectors or partial identities, on noun phrases. However, it is clear that the connectives do not follow this general principle: when we use {  \em and} to concatenate two sentences (or noun phrases, adverbs, \&c.), we do not expect to lose any information about the constituents in this conjunction. 

As a trivial example, consider taking some body of text, and replacing each full stop (period) by ``and''. Although legibility will rapidly be lost, it would be difficult to claim that any meaning or content has been erased. Thus, the element 
\[ \xymatrix{
I \ar[rr]^<<<<<<<<<{{   and}_X} && [X\rightarrow [X\leftarrow X]]  } \]
is the name of an arrow in $\M (X,[X\leftarrow X])$ that is  {\em information-preserving} in the sense laid out in Section \ref{SPpreserving} and Section \ref{namedforgetting}. We consider the implications of this shortly, but first use some abstract category theory to simplify the types of arrows being named.

\subsection{Revisiting types of connectives}
In order to make a {\em considerable} simplification of the resulting theory, we now make the assumption that the left evaluation arrow and the right evaluation arrow are identical (at least, up to some canonical symmetry isomorphism). Although there is no decisive linguistic justification for this in general\footnote{although since the connectives  {  \em and} and {  \em or} appear to be symmetric, this assumption is indeed justified for the particular examples we consider.}, it is certainly satisfied by {\em compact closed categories} \cite{KL},  which feature heavily in models of linguistics and meaning such as the vector spaces as used in distributional semantics,  the more general models of meaning of \cite{CSC}, and purely grammatical models such as Lambek pregroups \cite{JL}. 

Given this assumption, we may appeal to the defining equations of monoidal closure, from Definition \ref{monoidalclosed}, and -- up to isomorphism -- replace elements of type $\M((I,[X\rightarrow [X\leftarrow X]])$ by elements of type $\M(I,[X\otimes X\rightarrow X])$.

Thus (up to some canonical isomorphism that we elide in the following sections), a polymorphic connective such as {  \em and} determines a family of elements 
\[  {  and}_{(X)} \in \M( I , [X\otimes X\rightarrow X] ) \]
where $X$ ranges over various objects, including  $\{ S,NP , AD, TV  ,\ldots \} $. Further, as demonstrated in Section \ref{forgettingAnd} above, in each case $  and_X$ is the name of some arrow $\widetilde{  and_X} \in \M(X\otimes X,X)$ that preserves generalised inner products, and thus (from Lemma \ref{IPPinverse}) has a left inverse given by its dagger.

\subsection{Do arrows named by connectives have a right inverse?}
In Section \ref{forgettingAnd} above, we made the case that the object-indexed family of arrows named by the connective {  \em and} (and, quite possibly, other binary connectives) are information-preserving, in the sense that they are {\em isometries} -- i.e their adjoint is a left inverse, and thus they preserve the generalised inner product. In fact, it is easy to make a case that their adjoint should be a {\em two-sided} inverse, and they are thus unitary\footnote{We also observe that, even if this assumption should prove to be unfounded, the resulting mathematical structures will be almost identical; should we be forced to deal with some isometry $c_X\in \M(X\otimes X\rightarrow X)$ that has a left inverse, but not a right inverse, the appropriate mathematical tool will prove to be the {\em Karoubi envelope}, or {\em splitting idempotents} construction \cite{MCL}, as applied to exactly this situation in \cite{PH98,PH02}. }. The justification for this is (for the connective {  and}, in the case of the sentence object $S$) is that, given some sentence $W\in \C(I,S)$, we can always find some pair of sentences $U,V\in \C(I,S)$ such that the (evaluation of the) sentence $U\otimes {  and } \otimes V$ has exactly the same intended meaning as $W$.  A similar argument can be made for other objects in $\M$, and for other binary connectives.


\subsection{Frobenius algebras and self-similarity}\label{sweetFA}
We have seen that, for every polymorphic connective $c$ and appropriate object $X\in Ob(\M)$,  there exists some isomorphism $\widetilde{c_X} \in \M(X\otimes X,X)$ whose inverse is its dual $\widetilde{c_X} ^{-1} = \widetilde{c_X} ^\dagger \in \M(X\rightarrow X\otimes X)$, and thus $X\cong X\otimes X$ i.e. the object $X$ is self-similar in the sense of \cite{PH98,PH99}. The question we now address is whether, at least up to canonical isomorphism, this self-similarity gives rise to a {\em Frobenius algebra} structure at each of these objects in the category $\M$. 

Frobenius algebras in categories, definitions, diagrammatics, various special cases and applications and well covered in other chapters, so the following exposition is brief. In particular, we refer to Chapter 7 for more detailed theory, and refer to Chapter 1 for a suitable string-diagram formalism.

\begin{definition}\label{FA}
A {\em Frobenius algebra} in a monoidal category $(\C ,\otimes ,I)$ consists of a monoid structure $(\nabla:S\otimes S \rightarrow S,\bot : I\rightarrow S)$ and a comonoid structure  $(\Delta :S\rightarrow S \otimes S, \top : S \rightarrow I)$ at the same object, where the monoid / comonoid pair satisfy the {\bf Frobenius condition} 
\[ \Delta \nabla = (1_S \otimes \Delta) (\nabla \otimes 1_S) \ \in \C(S\otimes S, S\otimes S) \]
Expanding out the definitions of a monoid and a comonoid structure, we have:
\begin{enumerate}
\item {\bf (associativity)} $\nabla(1_S\otimes \nabla) = \nabla(\nabla\otimes 1_S)\in \C(S\otimes S \otimes S , S)$.
\item {\bf (co-associativity)} $(\Delta \otimes 1_S)\Delta = (1_S \otimes \Delta)\Delta  \in \C(S,S\otimes S \otimes S)$.
\item {\bf (unit)} $\nabla(\bot \otimes 1_S) = \nabla (1_S\otimes \bot)$.
\item {\bf (co-unit)} $(\top \otimes _S)\Delta = 1_X\otimes \top) \Delta$.
\end{enumerate}
\end{definition}

An immediate observation is that the above axioms for the monoid / comonoid structure ignore coherence isomorphisms. In particular, they assume {\em strict associativity} -- or at least, ignore the role of associativity isomorphisms. The same also holds for the Frobenius condition, since $1_S\otimes \Delta \in \C(S\otimes S , S\otimes (S\otimes S))$ whereas $\nabla \otimes 1_S\in \C((S\otimes S)\otimes S, S\otimes S)$. With this in mind, we make the following definition:

\begin{definition}\label{laxFrobDef}
A {\bf lax Frobenius algebra} in a monoidal category $(C,\otimes ,I)$ is defined to be an object $S\in Ob(C)$ along with arrows
\[ \Delta\in \C(S,S\otimes S) \ , \ 
\nabla \in \C(S\otimes S,S) \ , \  \top \in \C(S , I) \ , \ 
\bot \in \C(I,S) 
\] 
that satisfies the axioms of Definition \ref{FA} above, {\em up to canonical coherence isomorphisms}. 

It is also sometimes useful to consider structures that satisfy all the axioms for a Frobenius algebra -- whether lax or strict --  except those relating to the unit object (i.e. the existence of the arrows $\top,\bot$, and axioms 3.-4. above). Such structures $(S, \Delta ,\nabla)$ are called {\bf  unitless Frobenius algebras}. 
These are particularly relevant when working with the {\em unitless monoidal categories} of definition \ref{Spowers} onwards. 
\end{definition}

Our claim, to be justified over the following sections, is that the arrows named by connectives do indeed provide (lax, unitless) Frobenius algebras in the category $\M$. However, the details of the exact canonical coherence isomorphisms required are subtle --- and quite possibly controversial; we first need an in-depth investigation of the categorical structure of self-similarity.

\section{Self-similarity, categorically}\label{catselfsim}
{\em In this section, and the following sections, we do {\em not} appeal to MacLane's coherence theorem for associativity, and treat all associativity isomorphisms as though they were strict identities. For justification, we refer to Isbell's argument (quoted by MacLane in \cite{MCL} as justification for introducing associativity up to isomorphism) and give an updating of Isbell's argument to a more general setting in Appendix \ref{isbell}.}
\begin{definition}
Let $(\C,\otimes,I)$ be a monoidal category. A {\bf self-similar structure} $(S,\code,\decode )$ is defined to be an object 
$S\in Ob(\C)$, together with two mutually inverse arrows
\begin{itemize}
\item {\bf (code)} $\code\in \C(S\otimes S, S)$.
\item {\bf (decode)}  $\decode \in \C(S,S\otimes S)$.
\end{itemize}
satisfying $\decode\code = 1_{S\otimes S}$ and $\code \decode= 1_S$, so the following diagram commutes.
\[ \xymatrix{
 &  {S\otimes S} \ar@/^12pt/[rr]^{\code}  \ar@{}[l]^>>>>>>>>>>{}="a"    & &  S \ar@/^12pt/[ll]^{\decode} \ar@(ur,dr )[]^{1_S}
\ar@(dl,ul)^{1_{S\otimes S} } "a" ; "a"
} 
\]
When there is a self-similar structure at some object $S\in Ob(\C)$, we say (using the terminology of \cite{PH98,PH99}) that $S$ is a {\bf self-similar object}. Note that there may be many distinct self-similar structures at the same object.

When $(\C,\otimes, (\ )^\dagger)$ is a dagger monoidal category, and $\code=\decode^{-1}=\decode^\dagger$, we say that $(S,\code,\decode )$ is a {\bf dagger- self-similar structure}.\end{definition}
Motivating examples include the natural numbers $\mathbb N$ in various categories (relations, partial functions, partial reversible functions, \&c.) with respect to various monoidal tensors (Cartesian product, disjoint union). Other examples arise in the study of fractals (the Cantor set  \cite{PH98}, and fractals in general \cite{blatant}), logical models such as Scott's celebrated domain-theoretic models of the untyped lambda calculus (see \cite{LS} for a categorical exposition), inverse semigroups and tilings \cite{MVL,KeLa}, The Thompson groups \cite{MVL07}, and the Cuntz $C^*$ algebras \cite{JC}.

Although there is a close connection between such self-similar structures and the canonical coherence isomorphisms of a monoidal category \cite{PH02}, we emphasise that for any given object $S$, there are generally many self-similar structures. Simple cardinality arguments demonstrate that the set of bijections $\{ f: \mathbb N \rightarrow  \mathbb N \uplus \mathbb N \}$ is uncountable; this is expanded on in Appendix \ref{Hmonoid}, where an explicit correspondence between interior points of the Cantor set and order-preserving bijections from $\mathbb N$ to $\mathbb N \uplus \mathbb N$ is given. 

Despite this, the maps between self-similar structures are particularly simple:
\begin{definition}
Given two self-similar structures $(S,\code_1,\decode_1)$ and $(S,\code_2,\decode_2)$ at some object $S$ of a symmetric monoidal category $(\C,\otimes )$, a {\bf morphism} between them is an arrow $u\in \C(S,S)$ such that the following diagram commutes:
\[ 
\xymatrix{
	& S\otimes S \ar[dl]_{\code_1} \ar[dr]^{\code_2} & \\
S \ar[rr]_{u} & & S
}
\]
\end{definition}

\begin{proposition}
Let $u : (S,\code_1,\decode_1) \rightarrow (S,\code_2,\decode_2)$ be a morphism of self-similar structures.  Then $u:S\rightarrow S$ is the isomorphism given by $u=\code_2\decode_1$.
\end{proposition}
\begin{proof}
By definition of a self-similar structure, the following diagram  commutes:
\[ 
\xymatrix{	S\otimes S \ar[r]^{\code_1} \ar[d]_{\code_2} & S \ar[d]^{\decode_1} \ar[dl]|u \\
	S & S\otimes S \ar[l]_{\code_2 }
	} \]
and hence $u=\code_2\decode_1$. When these are dagger self-similar structures, it is also trivially unitary.
\end{proof}
Thus, with this definition of morphism, self-similar structures at some object $S\in Ob(\C)$ form a {\em skeletal category}, where there is exactly one arrow between any two objects.

\subsection{The generalised convolution functor}
Given an arbitrary object of a monoidal category $(\C , \otimes )$, it generates a subcategory of $\C$ in the obvious way:

\begin{definition}\label{Spowers}
Let $T$ be an arbitrary object of a monoidal category $(\C,\otimes, I )$. We define the category $T^\otimes$ {\bf generated by} $T$ and $\otimes$ to be the wide subcategory of $\C$ with the following inductively defined objects:
\begin{itemize}
\item $T \in Ob(T^\otimes )$.
\item Given $A,B\in Ob(T^\otimes )$, then $A\otimes B\in Ob(T^\otimes)$.
\end{itemize}
It is immediate that  $T^\otimes$ is closed under the monoidal tensor on both arrows and objects, and hence has all the structure of a monoidal category apart from the unit object $I$. Such categories are called {\bf unitless monoidal categories}. 
\end{definition}
Unitless monoidal categories are (trivially) not well-pointed. However -- as in the above example -- they may arise as subcategories of well-pointed categories.

\begin{proposition}\label{codedecode}
Given a self-similar structure $(S,\code,\decode )$ in a monoidal category $(\C,\otimes )$, then 
for every $X\in Ob(S^\otimes)$,
 there exists isomorphisms 
\[  \decode_{X} :S\rightarrow X \ \ , \ \  \code_{X}=\decode_X^{-1}: X \rightarrow S \]
\end{proposition}
\begin{proof}
We give these isomorphisms inductively:
\begin{itemize}
\item $\decode_S = 1_S = \code_S$
\item $\decode_{X\otimes Y} =  (\decode_X \otimes \decode_Y)\decode$
\item $\code_{X\otimes Y} =  \code(\code_X \otimes \code_Y)$
\end{itemize}
It is straightforward to verify that $\decode_A \in \C(S, A)$ and $\code_A\in \C(A, S)$ are isomorphisms, and each others inverse. Similarly, when $(S,\code,\decode )$ is a dagger self-similar structure then $\code_A=\decode_A^\dagger \in \C(A,S)$, for all $A\in Ob(\C)$.
\end{proof}

For every self-similar structure $(S,\code,\decode)$ in some monoidal category $(\C,\otimes)$ there is an obvious functor from $S^\otimes$ to $\C (S,S)$, considered as a one-object category:

\begin{definition}\label{gencon}
Given a self-similar structure $(S,\code,\decode )$ in a monoidal category $(\C,\otimes)$, 
we define the {\bf generalised convolution} functor 
\[ \Phi_{\dc} :S^\otimes \rightarrow \C (S,S) \]
as follows:
\begin{itemize}
\item {\bf (Objects)} $\Phi_{\dc}(X)=S$ for all $X\in Ob(S^\otimes)$
\item {\bf (Arrows)} Given $f\in S^\otimes (X,Y)$, then $ \Phi_{\dc}(f) = \code_Y f \decode_X$, as shown below:
\[ \xymatrix{
X \ar[r]^f & Y \ar[d]^{\code_Y}\\
S\ar[u]^{\decode_X} \ar[r]_{\Phi_{\dc} (f)} & S \\ } \]
where $\code_Y:Y\rightarrow S$ and $\decode_X:S\rightarrow X$ are as in Proposition \ref{codedecode}.
\end{itemize}
\end{definition}

\begin{proposition}\label{phifunctor}
Given a self-similar structure $(S,\code,\decode )$ in a monoidal category $(\C,\otimes)$, 
the {\bf generalised convolution} $\Phi_{\dc} :S^\otimes \rightarrow \C (S,S)$ defined above is indeed a functor. Further, when  $(\C,\otimes)$ is a dagger monoidal category, and $(S,\code,\decode )$ is a dagger self similar object, then the functor $\Phi_{\dc}$ preserves the dagger, so 
\[ \Phi_{\dc}(f^\dagger) = \left( \Phi_{\dc}(f)\right)^\dagger \]
\end{proposition}
\begin{proof}
We refer to \cite{PH98} for proof that $\Phi_{\dc} :S^\otimes \rightarrow \C (S,S)$ is indeed a functor. Now assume that $(S,\code,\decode )$ is a dagger self-similar structure. By definition, for arbitrary $f\in S^\otimes (X,Y)$
\[ \left(  \Phi_{\dc}(f) \right)^\dagger\ = \ \left(  \code_Y f \decode_X   \right)^\dagger \ = \  \left(  \Phi_{\dc}(f) \right)^\dagger =   \decode_X  ^\dagger f^\dagger \code_Y^\dagger \]
By Proposition \ref{codedecode} above, $\code_Y^\dagger=\decode_Y$ and $\decode_X  ^\dagger=\code_X$, so 
$ \left( \Phi_{\dc}(f) \right)^\dagger = \code_X f^\dagger \decode_Y = \Phi_{\dc}(f^\dagger)$
as required.
\end{proof}\subsection{Untyped monoidal categories}
As well as being functorial $\Phi_{\dc} :S^\otimes \rightarrow \C(S,S)$ preserves many categorical properties, such as monoidal structures, categorical closure, categorical traces, etc. (\cite{PH99}). We briefly outline how this gives $\C(S,S)$ the structure of a (one-object) monoidal category:

\begin{definition} \label{internaltensor}
Given  a self-similar structure $(S,\code,\decode )$ in a monoidal category $(\C,\otimes)$, we define the {\bf internal tensor} of $S$ determined by this self-similar structure to be the monoid homomorphism
\[ \_ \otimes_{\dc} \_ : \mathcal C(S,S)\times \mathcal C(S,S)\rightarrow \mathcal C(S,S) \]
given by the following generalised convolution:
\[ \xymatrix{
S\otimes S \ar[r]^{f\otimes g} & S\otimes S \ar[d]^{\code} \\
S\ar[u]^\decode \ar[r]_{f\otimes_{\dc} g} & S } \]
\end{definition}
We refer to \cite{PH98,PH99} for proof that this is a monoid homomorphism; this also follows from the fact that $f\otimes_{\dc} g:S\rightarrow S$ is, by definition, the image of $f\otimes g:S\otimes S \rightarrow S\otimes S$ under the generalised convolution functor $\Phi_{\dc} :S^\otimes \rightarrow \C (S,S)$.

We now demonstrate that $(\C(S,S),\otimes_{\dc})$ is a one-object unitless monoidal category.
\begin{theorem}\label{internaltensor_thm}
Given $S\in Ob(\C)$ as above, then there exists some $\tau_{\dc}\in \C(S,S)$ satisfying, for all $f,g,h\in \C(S,S)$, 
\begin{enumerate}
\item ({\bf Naturality}) ${\tau_{\dc} } . (f\otimes_{\dc} (g\otimes_{\dc} h)) \ = \ ((f \otimes_{\dc} g) \otimes_{\dc} h) . {\tau_{\dc}}$
\item ({\bf Pentagon condition}) $(\tau_{\dc} \otimes_{\dc} 1_S)\tau_{\dc}(1_S\otimes \tau_{\dc}) \ = \ \tau_{\dc}^2$
\end{enumerate}
\end{theorem}
\begin{proof}
Let $\tau_{\dc}\in \C(S,S)$ be defined as follows:
\[ \xymatrix{ 
S \ar[r]^>>>>>>\decode  \ar[d]_{\tau_{\dc}}	& S\otimes S \ar[rr]^<<<<<<<<{1_S \otimes \decode} & & S\otimes (S \otimes S) \ar[d]^{\tau_{S,S,S}} \\
S 						& S\otimes S \ar[l]^>>>>>>\code   & & ( S\otimes S)\otimes S \ar[ll]^<<<<<<<<{\code \otimes 1_S}
} 
\]
Either direct calculation,  or referring to \cite{PH98,PH99} will demonstrate that conditions 1. and 2. above are satisfied. Thus $(\C(S,S),\otimes_{\dc})$ satisfies all the axioms for a monoidal category, apart from the unit object. Note that in this category, associativity can only be up to isomorphism, and is never strict; forcing the identity $\tau_{\dc}=1_S$ will make $\C(S,S)$ collapse to an abelian monoid (See Appendix \ref{isbell}). \end{proof}

A simple corollary of this is the following:
\begin{corol}\label{identifiest}
Given a self-similar structure $(S,\code,\decode )$ in a monoidal category $(\mathcal C,\otimes,\tau_{\_\_\_ })$,  let $\Phi_{\dc} :S^\otimes \rightarrow \C (S,S)$ be the functor of Definition \ref{gencon}. Then for all objects $X,Y,Z,A,B,C\in Ob(S^\otimes)$,
\[ \Phi_{\dc}(\tau_{X,Y,Z}) = \Phi_{\dc}(\tau_{A,B,C}) \]
i.e. $\Phi_{\dc}$ maps all associativity isomorphisms of $S^\otimes$ to $\tau_{\decode\code}\in \C(S,S)$.
\end{corol}
\begin{proof}
This follows by the uniqueness of canonical isomorphisms in monoidal categories.\end{proof}

\begin{corol}\label{tisunitary}
Let  $(S,\code,\decode )$ be a dagger self-similar structure of some dagger-monoidal category $(\C,\otimes ,(\ )^\dagger)$. Then the one-object (unitless) monoidal category $(\C(S,S),\otimes_{\dc})$ is a (unitless) dagger-monoidal category. 
\end{corol}
\begin{proof}
We have seen that the functor $\Phi_{\dc}$ preserves the dagger operation. now choose arbitrary $A,B,C\in Ob(S^\otimes)$. By Corollary \ref{identifiest} above,
\[  \tau_{\dc}^\dagger =  \Phi_{\dc}(\tau_{X,Y,Z}^\dagger) =  \Phi_\dc (\tau_{A,B,C}^{-1}) = \tau_\dc ^{-1} \]
Thus $\tau_\dc$ is unitary, as required.
\end{proof}

As well as preserving the monoidal structure, the functor $\Phi_{\dc}$ preserves any symmetric monoidal structure. We outline the proof, and refer to \cite{PH98,PH99}  for details.

\begin{theorem}
Let  $(S,\code,\decode )$ be a self-similar structure of some symmetric monoidal category $(\C,\otimes ,t,s)$. Then the functor $\otimes_{\dc}$ of Definition \ref{internaltensor} above is symmetric, up to a natural isomorphism satisfying MacLane's hexagon condition. Further, when $(S,\code,\decode )$ is a dagger self-similar structure, then this canonical isomorphism is also unitary.
\end{theorem}
 \begin{proof}
 We define the arrow $\sigma_\dc\in \C(S,S)$ by the following convolution:
 \[ \xymatrix{
 S \otimes S \ar[r]^{\sigma_{S,S}} & S\otimes S \ar[d]^\code \\
 S\ar[u]^\decode \ar[r]_{\sigma_\dc} & S 
 } \]
 Equivalently, $\sigma_\dc=\Phi_\dc (\sigma_{S,S})$. The functoriality of $\Phi_\dc$ implies
 \begin{itemize}
 \item {(\bf Naturality)}  $\sigma_\dc(f \otimes_\dc g) = (g\otimes_\dc f)\sigma_\dc$ for all $f,g\in \C(S,S)$.
 \end{itemize}
 Either direct calculation or reference to \cite{PH98,PH99} will also demonstrate the following:
 \begin{itemize}
 \item {\bf (MacLane's hexagon)} $\tau_\dc \sigma_\dc \tau_\dc =(\sigma_\dc \otimes_\dc  1_S)\tau_\dc  (1_S\otimes_\dc  \sigma_\dc )$.
 \end{itemize}
 Uniqueness of canonical isomorphisms will (in the same manner as Corollary \ref{identifiest}) demonstrate that 
 $\Phi_\dc (\sigma_{A,B}) = \sigma_\dc$ for arbitrary $A,B\in Ob(S^\otimes)$, and therefore, using almost identical reasoning to Corollary \ref{tisunitary}, when $(S,\code,\decode)$ is a dagger self-similar structure the arrow $\sigma_\dc\in \C(S,S)$ is also unitary.
 \end{proof}

\subsection{Untyped and polymorphically typed systems: a discussion}
The functor $\Phi_\dc$ of Definition \ref{gencon} may be seen to be a general {\em type-erasing} construction; it maps various categorical structures to one-object (i.e. untyped) analogues of the same structures. Examples include, but are certainly not limited to, symmetric monoidal structures, categorical closure (including compact closure and Cartesian closure), categorical traces, projections and injections, \&c. \cite{PH98,PH99}.  

The immediate question must then be: 
\begin{quotation}{\em Why are we spending so much time developing a type-erasing procedure, when the whole point of the linguistics project is to -introduce- types into models of meaning?}\end{quotation}
Recall that the starting point for our investigation of self-similarity was the (polymorphically typed) connectives; linguistic arguments were used to demonstrate that many of the objects corresponding to types in our models of meaning must be self-similar, with the self-similarity exhibited by the arrows named by connectives.  Partly, therefore, the existence of such structures is forced upon us by the typing of connectives and their intended interpretation in some categorical model of meaning. However, there is a more fundamental justification; so far we have simply treated the polymorphically typed connectives as arrows parametrized by some class of objects. If we were to take a more foundational approach and look for models based on models of (for example) System F, we would discover a close connection between polymorphic typing and such a type-erasing procedure. 

In models of polymorphic lambda calculus and related systems, the underlying categories commonly have a single object; the types of the logical system are built up from certain families of arrows, satisfying a `biorthogonality' relation. Although this is far beyond the scope of this paper, we refer to \cite{HyS} for an interesting point of view, and details and references for this kind of approach. From a linguistic point of view, we simply remark that it is perhaps not so surprising that the polymorphically typed terms are exactly those that do not lose information --- inevitably and counterintuitively leading to such a type-erasing procedure. 

\section{Self-similarity and lax Frobenius algebras}
We have now developed the categorical machinery that enables us to justify the claim made at the end of Section \ref{sweetFA} that self-similar structures form (lax unitless) Frobenius algebras (Definition \ref{laxFrobDef}) i.e. they satisfy the axioms of  Definition \ref{FA} (excluding those based on the unit object) up to canonical coherence isomorphisms.

\begin{theorem}\label{selfsim=fa}
Let $(S,\code ,\decode)$ be a self-similar structure in some monoidal category $(\C , \otimes , \tau )$, and let $(\C(S,S), \otimes_\dc , \tau_\dc)$ be the corresponding one-object unitless monoidal category described in Theorem \ref{internaltensor_thm}. Then the following conditions are satisfied
\begin{enumerate}
\item {\bf (unitless monoid)} The arrow $\code \in \C(S\otimes S ,S)$ is associative, up to the canonical associativity isomorphisms $\tau_{S,S,S},\tau_\dc$,
\item {\bf (unitless comonoid)} The arrow $\decode\in \C(S,S\otimes S )$ is co-associative, up to the canonical associativity isomorphisms $\tau_{S,S,S}^{-1},\tau_\dc$,
\item {\bf (unitless Frobenius condition)}The pair of arrows $\code \in \C(S\otimes S ,S)$ and  $\decode\in \C(S, S\otimes S)$ satisfy the Frobenius condition of Definition \ref{FA}, up to the canonical associativity isomorphisms $\tau_{S,S,S}^{-1} ,\tau_\dc$,
\end{enumerate}
and hence the self-similar structure $(S,\code ,\decode)$ is a (unitless, lax) Frobenius algebra.
\end{theorem}
\begin{proof}
Almost by definition, the following diagrams may be seen to commute:
\begin{enumerate}
\item {\bf (Lax associativity)}
 \[ \xymatrix{ 
S \ar[r]^>>>>>>\decode  \ar[d]_{\bf  \tau_\dc}	& S\otimes S \ar[rr]^<<<<<<<<{1_S \otimes \decode} & & S\otimes (S \otimes S) \ar[d]^{\bf \tau_{S,S,S}} \\
S \ar[r]_>>>>>>\decode						& S\otimes S \ar[rr]_<<<<<<<<{\decode \otimes 1_S} & & ( S\otimes S)\otimes S 
} 
\]
\item  {\bf (Lax co-associativity)}
\[ \xymatrix{ 
S\otimes (S\otimes S) \ar[rr]^{1_S\otimes \code} \ar[d]_{\bf \tau_{S,S,S}} & & S\otimes S \ar[r]^{\code} & S \ar[d]^{\bf \tau_\dc} \\
(S\otimes S)\otimes S \ar[rr]_{\code \otimes 1_S} & & S\otimes S \ar[r]_{\code} & S \\
} 
\]
\item  {\bf (Lax Frobenius condition)}
\[ \xymatrix{
S \otimes S \ar[rr]^{\decode \otimes 1_S}\ar[d]_{\decode \circ {\bf \tau_\dc } \circ \code} 	& 	& (S\otimes S )\otimes S \ar[d]^{\bf \tau^{-1}_{S,S,S}}  \\
S \otimes S 													&	& S\otimes (S\otimes S)  \ar[ll]^{1_S\otimes \decode} \\
} 
\]
\end{enumerate}
\end{proof}

\begin{corol}\label{SSdaggerFA}
Let $(S,\code ,\decode)$ be a dagger self-similar structure in some dagger monoidal category $(\C , \otimes , \tau (\ )^\dagger)$. Then the dagger self-similar structure $(S,\code ,\decode)$ is --- up to the same canonical coherence isomorphisms listed above -- a (unitless, lax) dagger Frobenius algebra.
\end{corol}
\begin{proof}
This is a simple corollary of Theorem \ref{selfsim=fa} above, Proposition \ref{phifunctor}, and Corollary \ref{tisunitary}.
\end{proof}

\begin{remark}\label{isitallbollocks?}{\bf The r\^ole of coherence isomorphisms in Theorem \ref{selfsim=fa}}\\
From a certain point of view, the proof of Theorem \ref{selfsim=fa} seems to be cheating, in that the diagrams used to {\em prove} associativity \&c.  up to isomorphism are minor variants of those used to {\em define} the associativity isomorphism $\tau_\dc$. In particular, it seems highly unsurprising that associativity and co-associativity hold up to isomorphism, since after all, the code and decode arrows are both themselves isomorphisms. What rescues this from being a triviality is that $\tau_\dc$ is natural and satisfies MacLane's pentagon condition (Theorem \ref{internaltensor_thm}).  Therefore, in a very strong sense, it is exactly a {\em canonical coherence isomorphism}. 

Even so, the fact that two, rather than one, canonical isomorphisms are required may be considered to be pushing the definition of `lax' too far -- especially since they are, technically, coherence isomorphisms for two distinct monoidal categories.  However, these categories and monoidal tensors are not arbitrary; instead, one is a wide subcategory of the other, and the distinct monoidal tensors and coherence arrows are mutually definable by a generalised form of convolution. What is really required is a coherence theorem for this very special situation!

In the absence of a full coherence theorem covering such a situation, we are forced to rely on the familiar coherence theorems of MacLane for monoidal categories. Ultimately, however, the acid test must be whether such structures behave in a similar manner to more familiar Frobenius algebras. We are thus lead into particular examples of Frobenius algebras and their applications, to demonstrate that this is indeed the case.
\end{remark}

\section{Classical structures}\label{CSsection}
A particular form of Frobenius algebra, used heavily in categorical quantum mechanics, is the {\em classical structure} \cite{CP}.  When modelling quantum phenomena in abstract categories (e.g. as in \cite{CP,DP}) the notion of a classical structure is fundamental in ways beyond the scope of this chapter --- although fundamental to other chapters in this volume. Instead, we consider their behaviour in a particular concrete category.

In the symmetric dagger monoidal category $({\bf Hilb_{FD}},\otimes (\ )^\dagger )$ of finite dimensional Hilbert spaces with tensor product and Hermitian adjoint, a classical structure at an object $H\in Ob({\bf Hilb_{FD}})$ is exactly an orthonormal basis for the Hilbert space $H$. Thus, each classical structure at on object determines {\em matrix representations} for linear maps on this object. The connection between orthonormal bases and measurements is clear, and the (physically reasonable) unitary maps arise as isomorphisms of classical structures, or equivalently as changes of basis. 

\begin{definition}\label{CS}
A {\bf classical structure} $(S,\Delta , \nabla ,\top,\bot)$  in a dagger monoidal category $(\C ,\otimes, I , (\ )^\dagger)$ is defined to be a dagger Frobenius algebra satisfying, for all $f,g\in \C(S,S)$,
\begin{enumerate}
\item {(\bf Commutativity)} $(f\otimes g)\Delta = (g\otimes f)\Delta$
\item {(\bf Co-Commutativity)} $\nabla(f\otimes g) = \nabla(g\otimes f)$
\item {(\bf The Classical structure condition)} $\nabla \Delta=1_S$
\end{enumerate}
\end{definition}
Note that Definition \ref{CS} above is somewhat over-axiomatised. In particular, as shown in \cite{DP}, the classical structure condition on any Frobenius algebra $(S,\Delta , \nabla ,\top,\bot)$ in a dagger monoidal category will imply that $(S,\Delta , \nabla ,\top,\bot)$ is a dagger Frobenius algebra.  Further, once the identity $\Delta=\nabla^\dagger$ is satisfied, then commutativity and co-commutativity are equivalent. We have deliberately taken this over-axiomatised route, since it is not clear how many of these implications will survive the passage to the lax unitless version of the above structures. 
\begin{definition}\label{LCS}
We define a {\bf lax classical structure} in a dagger monoidal category $(\C ,\otimes , I, (\ )^\dagger)$ to be a lax unitless dagger Frobenius algebra satisfying conditions 1.-3. of Definition \ref{CS} above, up to canonical coherence isomorphisms. 
\end{definition}

\begin{theorem}
Let $(S,\code ,\decode)$ be a dagger self-similar structure in a symmetric dagger monoidal category $(\C ,\otimes , \tau ,\sigma )$, and let $(\C(S,S), \otimes_\dc , \tau_\dc,\sigma_\dc )$ be the corresponding one-object unitless dagger symmetric monoidal category  Then:
\begin{enumerate}
\item  $(S,\code ,\decode)$ is a lax unitless dagger Frobenius algebra,
\item The arrow $\decode\in \C(S,S\otimes S)$ is co-commutative, up to the canonical coherence isomorphisms $\sigma_{S,S},\sigma_\dc$,
\item The arrow $\code\in \C(S\otimes S,S)$ is commutative, up to the canonical coherence isomorphisms $\sigma_{S,S},\sigma_\dc$,
\item The classical structure condition $\code\decode =1_S$ holds, strictly.
\end{enumerate}
and hence $(S,\code ,\decode)$ is a 
 lax classical structure.
\end{theorem}
\begin{proof}{\em 
As in Theorem \ref{selfsim=fa}, the following proof is almost by definition. We refer to Remark \ref{isitallbollocks?} for a discussion of the issues around this, and Section \ref{daoud} below for justification by example of why this is reasonable.}
\begin{enumerate}
\item We refer to Theorem \ref{selfsim=fa}  and Corollary \ref{SSdaggerFA} for a proof that $(S,\code,\decode)$ is a lax unitless Frobenius algebra. 
\item By definition of $\sigma_\dc$ and naturality of canonical coherence isomorphisms, the following diagram commutes, for all $f,g\in \C(S,S)$:
\[ \xymatrix{ 
S \ar[r]^\decode \ar[d]_{\sigma_\dc} & S\otimes S \ar[r]^{f\otimes g} & S \otimes S \\
S \ar[r]_{\decode} & S\otimes S\ar[r]_{g\otimes f} & S\otimes S \ar[u]_{\sigma_{S,S } }
}
\]
and hence by naturality of both  $\sigma_{S,S}$ and $\sigma_\dc$,
\[ (f\otimes g)\decode \cong (g\otimes f)\decode \]
up to canonical coherence isomorphisms.
\item Similarly to 2.,  since $\sigma_{S,S}=\sigma^{-1}_{S,S}=\sigma_{S,S}^\dagger$, $\sigma^{-1}=\sigma^\dagger$, and $\code^\dagger=\decode$, the commutativity of the above diagram for all $f,g\in \C(S,S)$ implies the commutativity of the following diagram:
\[ \xymatrix{ 
S\otimes S \ar[r]^{f\otimes g} \ar[d]_{\sigma_{S,S}} & S\otimes S \ar[r]^{\code} & S \\
S\otimes S \ar[r]_{g\otimes f} & S\otimes S \ar[r]_{\code} & S \ar[u]_{\sigma_{\dc}}
}
\]
and thus $\code\in \C(S\otimes S,S)$ is commutative up to canonical coherence isomorphisms.
\item The defining equation of a self-similar structure is:
\[ \code\decode = 1_S \ \ , \ \ \decode \code = 1_{S\otimes S} \] 
-- a  stronger (i.e. two-sided) version of the classical structure condition.
\end{enumerate}
\end{proof}

It may be objected again (as in Remark \ref{isitallbollocks?}) that using canonical isomorphisms from two distinct settings is pushing the definition of a lax structure too far. In the absence of a full coherence theorem relating the monoidal tensor of a category (\& its canonical isomorphisms) with the internal tensor at an object (\& its canonical isomorphisms), this is worth considering. 

However, we now present an intriguing example  where distinct self-similar structures at some object $S$ within a symmetric monoidal category  determine distinct matrix representations for arrows on this object.

\section{A self-similar structure familiar in logic (and linguistics)}\label{daoud}
One of the best-studied self-similar structures, at least in certain logical communities, is the {\em dynamical algebra}. There are many wide-ranging applications, including the pure untyped lambda calculus \cite{DR}, linear logic and the Geometry of interaction \cite{GOI}, combinatory logic \cite{AHS}.

In \cite{DR}, the dynamical algebra is introduced as the monoid semiring $P[\mathbb N]$, where the monoid $P$ may be defined in terms of generators and relations, as follows:
\[ P \ = \ \langle p,q,p',q' : pp'=1=qq' \ , \ pq'=0=qp' \rangle \]
(This is of course, the inverse {\em polycyclic monoid} of \cite{NP}).

Other definitions vary in the precise notion of summation used, often restricting or generalising the summation via reference to some representation (e.g. analytic convergence in infinite-dimensional Hilbert space \cite{GOI}, suprema in the natural partial order of an inverse category \cite{PH98}, the $\Sigma$-monoid axiomatisation of \cite{HA,AHS}, \&c.) However, in every case, the representations given are self-similar structures in a monoidal category (a purely categorical explanation of how the polycyclic monoid arises from self-similarity is given in \cite{PH99}).

A particularly well-studied example is the monoidal category of partial isomorphisms on sets, with disjoint union $({\bf pInj} , \uplus)$. Here, any bijection $\mathbb N \cong \mathbb N \uplus \mathbb N$ gives rise to an embedding of the dynamical algebra in $\bf pInj(\mathbb N ,\mathbb N)$ \cite{PH98,PH99}. 

The key feature at this point is that --- even though it does {\em not} have a  coproduct --- the monoidal category $(\bf pInj,\uplus)$ admits matrix representations, as observed by many authors \cite{PH98,MVL,AHS}. Thus, any isomorphism $\mathbb N \cong \mathbb N \uplus \mathbb N$ allows us to give $(2\times 2)$ matrix representations to partial bijections on $\mathbb N$. However, matrix representations for these arrows are not unique; instead, each matrix representation uniquely determines, and is determined by,  such an isomorphism. The details of how these dagger self-similar structures act like bases for such matrix representations, allowing for many properties more familiar from linear algebra such as matrix representations, changes of basis, diagonalisations, mutual diagonalisation, \&c. are the subject of a forthcoming paper.\\

\noindent Finally, for readers who fear that we have by now strayed way too far from any possible linguistic interpretation, we refer to the recent rediscovery of structures isomorphic to the dynamical algebra in linguistic models of meaning \cite{DC1,CWL}.

\section{Acknowledgements}
The author wishes to thank:
Samson Abramsky,  for discussions and assistance in most of the topics covered in this chapter; 
Daoud Clarke for a linguistic point of view, and particularly the linguistic interpretation of the models touched on in Section \ref{daoud};
Steve Clarke, for putting up with what are probably --- from a linguists point of view -- very simple questions; 
Bob Coecke for discussions on category theory and coherence, and their connections with both linguistics and quantum mechanics;
Mark Lawson, for self-similarity from an algebraic \& semigroup-theoretic viewpoint; 
Phil Scott for many discussions and assistance on the logical interpretations of closed categories, and interpretations of the Geometry of Interaction; 
Prakash Panangaden for both a physicist's point of view, and the intuitive code / decode terminology for the arrows in a self-similar structure. 

Finally, thanks are due to Mehrnoosh Sadrzadeh, from both a linguistic / categorical perspective, and for the considerable organisation that made the book in which this work appears possible.

\bibliographystyle{plain}
\bibliography{forget_refs}

\appendix

\section{Order-preserving bijections $\mathbb N \cong \mathbb N \uplus \mathbb N$ as interior points of the Cantor set}\label{Hmonoid}
We now give an explicit illustration of how the points of the Cantor set (excluding a distinguished subset of measure zero) may be interpreted as order-preserving bijections exhibiting the self-similarity of the natural numbers (with respect to disjoint union; however, a very similar construction applies to the natural numbers with Cartesian product -- this is left as an exercise). 

\begin{definition}
The Cantor set is defined to be the set of all one-sided infinite binary strings, $\mathscr C = \{ 0,1\}^\omega$ or equivalently, the set ${\bf Fun}(\mathbb N , \{ 0,1\})$ of all functions from $\mathbb N$ to $\{ 0,1\}$. A point $a=a_0 a_1 a_2 a_3 \ldots$ of the Cantor set is called a {\bf boundary point} if there exists some $K\in \mathbb N$ such that, for all $L\geq K$, $a_L=a_K$. We denote the set of all boundary points by $\mathscr C_\mathscr B$. 
The complement of the boundary is called the {\bf interior}, denoted $\mathscr C_\mathscr O = \mathscr C \setminus \mathscr C_\mathscr B$. Members of $\mathscr C_\mathscr O$ are called {\bf interior points} or {\bf balanced functions}.  Note that the set of boundary points is a countable subset of $\mathscr C$, whereas its complement, the set of interior points,  has the same cardinality as the Cantor set itself
\end{definition}

\begin{theorem}\label{closedmonoid}
Each interior point of the Cantor set $\eta\in \mathscr C_\mathscr O$ uniquely determines and is determined by an order-preserving bijection $\widetilde{\eta}:\mathbb N \rightarrow \mathbb N \uplus \mathbb N$.
\end{theorem}
\begin{proof}
By construction, the balanced function $\eta:\mathbb N \rightarrow \{ 0,1\}$ divides $\mathbb N$ into two disjoint countably infinite subsets $\mathbb N =\eta^{-1}(0) \cup \eta^{-1}(1)$. Each of these subsets is totally ordered, with order inherited from $\mathbb N$ in the obvious way, so  $\mathbb N$ is divided into two disjoint countably infinite chains. This is illustrated by example in figure \ref{primechains}.

Each interior point $\eta:\mathbb N \rightarrow \{ 0 , 1\}$ uniquely determines and is determined by such a split into disjoint chains; $\eta$ is simply the indicator function for chain membership. As both chains are countably infinite, this therefore determines a bijection $\widetilde{\eta}:\mathbb N \rightarrow \mathbb N \uplus \mathbb N$ in the obvious way. Formally, taking $\mathbb N \uplus \mathbb N \stackrel{def.}{= } \mathbb N \times \{0\} \cup \mathbb N \times \{ 1\}$, we have 
\[ \widetilde{\eta}(n) = (x,\eta(n)) \ \mbox{where} \ x=\left\{ \begin{array}{lr} n-\sum_{j<n} \eta(j) & \eta(n)=0 \\  & \\ \sum_{j<n}\eta(j) & \eta(n)=1 \end{array}\right. \]
This is again illustrated in Figure \ref{primechains}.

By construction, $\widetilde{\eta}:\mathbb N \rightarrow \mathbb N \uplus \mathbb N$ is order-preserving, and since $\eta$ is simply the indicator function for the division of $\mathbb N$ into two disjoint countably infinite chains, all such bijections arise in this way. Finally, it is immediate that 
\[ \widetilde{\eta} = \widetilde{\mu} \Leftrightarrow \eta = \mu \] 
and so $\widetilde{(\ )}$ is bijective. 
\end{proof}

\begin{figure}[h]\begin{center}\caption{A balanced indicator function, two chains and a bijection}\label{primechains}
\[ {\small
\begin{array}{r|ccccccccccc}
n=			&0     	& 1     	& 2     	& 3      	& 4     	& 5      	& 6      	& 7     	& 8      	& 9      	& \cdots \\
\hline
p(n)= 		&0		& 0		& 1		& 1		& 0		& 1		& 0		& 1		& 0 		& 0		& \cdots \\
\hline
p^{-1}(0) \ : 	&0     	& 1     	& 	 	& 	 	& 4    	 & 	 	& 6      	& 	 	& 8      	& 9 	     	& \cdots \\ 
p^{-1}(1) \ : 	& 		& 	 	& 2     	& 3      	& 	 	& 5      	&  	 	& 7     	& 	 	& 		& \cdots \\
\hline
\widetilde{p}(n)=& (0,0)	& (1,0)	 & (0,1) 	& (1,1)      	& (2,0)	& (2,1)      	&  (3,0)	& (3,1)     	& (4,0) 	& (5,0)	& \cdots \\
\end{array}
}
\]
\end{center}
\end{figure}

The following special case is then immediate:
\begin{corol} Under the correspondence of Theorem \ref{closedmonoid}, the familiar Cantor pairing
\[ n \mapsto \left\{ \begin{array}{lr} 
\left(\frac{n}{2} , 0 \right) & n \ \mbox{ even} \\
& \\ 
 \left(\frac{n-1}{2} , 1 \right) & n \ \mbox{ odd} \\
 \end{array}\right. \]
 corresponds to the alternating point $0101010101 \ldots $  of the Cantor set, or equivalently the function $n\mapsto n \ (mod \ 2)$.
 \end{corol}
 
 Although the above example is simple,  we emphasise that Theorem \ref{closedmonoid} is about {\em all} interior points of the Cantor set, not simply the computable ones.  For example, we could consider an enumeration of Turing machines $\mathfrak T_0,\mathfrak T_1,\mathfrak T_2, \mathfrak T_3, \ldots$ and define the interior point $u\in \mathscr C_\mathscr O$ by $u(n)=0$ if $\mathfrak T_n$ halts, and $u(n)=1$ otherwise. By Theorem \ref{closedmonoid}, this also corresponds to an (uncomputable) order-preserving bijection $\mathbb N \rightarrow \mathbb N \uplus \mathbb N$.

It is entirely possible that, should we restrict ourselves to {\em computable} interior points of the Cantor set, we would be able to develop a consistent theory of {\em computable self-similar structures}. However, this is work that remains to be carried out.

\section{Isbell's argument in a general setting}\label{isbell}
In \cite{MCL}, MacLane introduces associativity up to isomorphism, by reference to a result of J. Isbell, on denumerable objects in the category $({\bf Fun} ,\times)$ of sets and functions with Cartesian product. Isbell demonstrated that, for any object $D\in Ob({\bf Fun})$, the existence of a bijection between $D$ and $D \times D$ would, in the presence of {\em strict} associativity, force all endomorphism arrows of $D$ to be identified (i.e. $D$ would necessarily be isomorphic to the unit object of $({\bf Fun} ,\times)$). 

This gave a very strong argument for the necessity of considering associativity up to isomorphism, rather than strict associativity; requiring strict associativity forces an identification of the  natural numbers $\mathbb N$ with the one-object set $\{ *\}$, in the category $({\bf Fun} ,\times)$.
A slight variation of Isbell's argument, in more modern language, demonstrates that this is not unique to the Cartesian closed category $\bf Fun$.

Let $S$ be a self-similar object of some symmetric monoidal category $(\C ,\otimes )$, and let the subcategory $S^\otimes$ be the full subcategory of $S$ given in Definition \ref{Spowers}. It is stated in this definition that $S^\otimes$ is a {\em unitless} monoidal category. This is correct; however, it is worth considering the objects of $S^\otimes$ to see why none of these act as the unit object for $S^\otimes$. In particular, for every object $X\in Ob(S^\otimes)$, it is a simple corollary of Proposition \ref{codedecode} that 
\[ X\otimes S \cong X \cong S\otimes X \]
The natural question then, is: {\em why is $S\in Ob(S^\otimes )$ \underline{not} the unit object?} The simple answer is that, although $X\otimes S \cong X \cong S\otimes X$, the arrows exhibiting these isomorphisms do not satisfy the coherence conditions given in \cite{MCL} --- in particular, their interaction with the other canonical isomorphisms  fails. 

However, when we ignore canonical isomorphisms and coherence conditions  this is no longer an obstacle, and we are forced to conclude that the category $(S^\otimes,\otimes)$ does indeed have a unit object; $S$ itself -- with all this implies about the endomorphism monoid of $S$ (see Proposition \ref{abstractScalars}). 

\end{document}